\documentclass[11pt]{article}
\ifdefined\pdfminorversion\pdfminorversion=7\fi

\usepackage[preprint,hyperref]{acl}

\usepackage{times}
\usepackage{latexsym}
\usepackage[T1]{fontenc}
\usepackage[utf8]{inputenc}
\usepackage{microtype}
\usepackage{inconsolata}
\usepackage{graphicx}

\usepackage{amsmath,amsfonts,bm}

\def\eqref#1{equation~\ref{#1}}

\def\1{\bm{1}}

\def\vtheta{{\bm{\theta}}}

\def\vx{{\bm{x}}}
\def\vy{{\bm{y}}}

\DeclareMathAlphabet{\mathsfit}{\encodingdefault}{\sfdefault}{m}{sl}
\SetMathAlphabet{\mathsfit}{bold}{\encodingdefault}{\sfdefault}{bx}{n}

\newcommand{\E}{\mathbb{E}}

\usepackage{algorithm}
\usepackage{algorithmic}

\usepackage{amsmath}
\usepackage{amssymb}
\usepackage{mathtools}
\usepackage{amsthm}
\theoremstyle{plain}
\usepackage{thm-restate}

\declaretheorem[name=Proposition]{proposition}
\declaretheorem[name=Lemma]{lemma}
\makeatletter
\let\c@proposition\c@theorem

\let\c@lemma\c@theorem

\makeatother

\usepackage{subcaption}
\usepackage{booktabs}
\usepackage{xltabular} %
\usepackage{multirow}
\usepackage{makecell}
\usepackage{listings}
\usepackage{tabularx}
\usepackage{placeins}
\usepackage{float} %
\usepackage{wrapfig}
\usepackage{adjustbox}

\usepackage{url}

\usepackage[capitalize,noabbrev]{cleveref}
\crefname{section}{Sec.}{Secs.}
\Crefname{section}{Section}{Sections}
\crefname{table}{Tab.}{Tabs.}
\Crefname{table}{Table}{Tables}
\crefname{figure}{Fig.}{Figs.}
\Crefname{figure}{Figure}{Figures}
\crefname{equation}{Eq.}{Eqs.}
\Crefname{equation}{Equation}{Equations}
\crefname{assumption}{Assumption}{Assumptions}
\Crefname{assumption}{Assumption}{Assumptions}
\crefname{theorem}{Theorem}{Theorems}
\Crefname{theorem}{Theorem}{Theorems}
\crefname{remark}{Remark}{Remarks}
\Crefname{remark}{Remark}{Remarks}
\crefname{proposition}{Proposition}{Propositions}
\Crefname{proposition}{Proposition}{Propositions}
\crefname{lemma}{Lemma}{Lemmas}
\Crefname{lemma}{Lemma}{Lemmas}

\newcommand{\appref}[1]{appendix~\labelcref{#1}}
\newcommand{\Appref}[1]{Appendix~\labelcref{#1}}

\newif\ifappendixcontents
\let\SavedAddcontentsline\addcontentsline
\renewcommand{\addcontentsline}[3]{%
  \ifappendixcontents\SavedAddcontentsline{#1}{#2}{#3}\fi}

\title{Hindsight-Privileged RLVR: \\ Single-Rollout Token Credit Without Group Sampling}

\author{
  \textbf{Wei Chen\textsuperscript{1}, Yubing Wu\textsuperscript{1}, Zhou Xu\textsuperscript{2}, Haoyue Zhang\textsuperscript{3}, Delu Zeng\textsuperscript{1}, Jiahao Wang\textsuperscript{4}} \\[2pt]
  \textbf{Junmei Yang\textsuperscript{1}, Pan tiewen\textsuperscript{5}, Jingfeng Liu\textsuperscript{6}, Zhiheng Zhou\textsuperscript{1}, Bo Li\textsuperscript{1}, Min Chen\textsuperscript{1}} \\[4pt]
  {\small\normalfont \textsuperscript{1}South China University of Technology\quad \textsuperscript{2}Tsinghua University} \\
  {\small\normalfont \textsuperscript{3}The Hong Kong University of Science and Technology\quad \textsuperscript{4}Noah Research Institute} \\
  {\small\normalfont \textsuperscript{5}Shenzhen Technology University\quad \textsuperscript{6}LinkSprite}
}

\begin{document}

\maketitle

\begin{abstract}
Reinforcement learning with verifiable rewards (RLVR) trains language models from checked task outcomes, but assigning a terminal reward to individual reasoning tokens or actions remains difficult. Group-relative methods such as Group Relative Policy Optimization (GRPO) sample several attempts per prompt, but lose their reward-based learning signal when all attempts receive the same outcome. We propose Hindsight-Privileged RLVR (HP-RLVR), a single-rollout method that turns a completed trajectory's own future into token-level update weights. The generating policy snapshot scores each eligible token under its causal history and under a hindsight context containing the realized future and verifier feedback. A normalized Jensen--Shannon divergence between the two views weights the reward relative to a historical baseline, yielding bounded token credit for both successful and failed attempts, with no reference solution or learned critic. In paired Qwen3-8B runs over a shared $100$-step schedule, HP-RLVR trains $1.21\times$, $1.46\times$, and $2.30\times$ faster than GRPO on mathematics, code generation, and ALFWorld, using one eighth the rollouts per prompt. Accuracy is comparable overall, with HP-RLVR ahead on mathematics at all three sizes and GRPO ahead on code and ALFWorld. In ALFWorld, $43.2\%$ (4B) and $49.0\%$ (8B) of GRPO's gradient-norm records are zero, versus none for HP-RLVR.
\end{abstract}

\section{Introduction}
\label{sec:intro}

\begin{figure}[t]
\centering
\includegraphics[width=\linewidth]{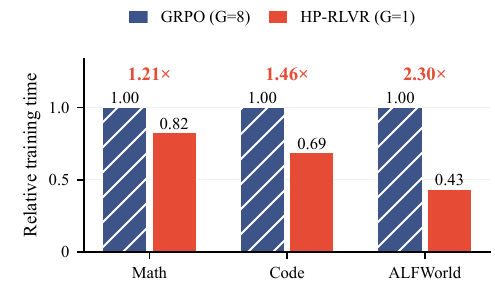}
\caption{Qwen3-8B training cost at four prompts per step and $100$ steps. HP-RLVR uses group size $G{=}1$, versus GRPO's $G{=}8$. Bars are relative to GRPO, with labels above each pair showing speedup.}
\label{fig:single_rollout_cost}
\end{figure}

\begin{figure*}[!t]
\centering
\includegraphics[width=\linewidth]{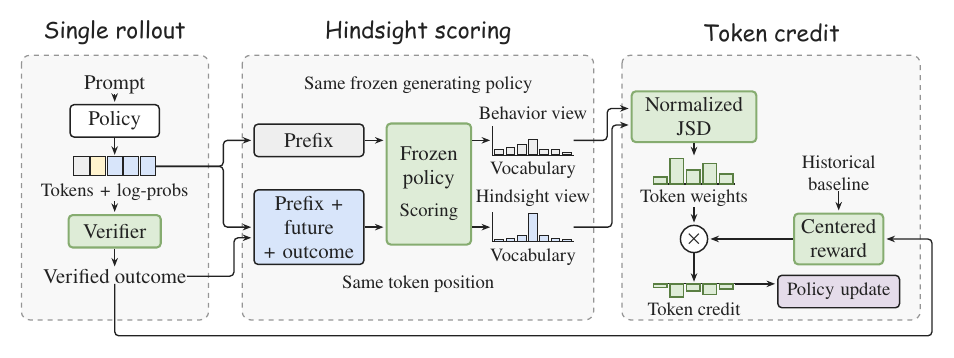}
\caption{Overview of HP-RLVR's single-rollout training pipeline with training-only hindsight. After generation, the frozen generating policy scores the trajectory under causal and hindsight contexts, without additional sampling. The paired profiles show full-vocabulary distributions at the same token position. Normalized JSD weights the historically centered verifier reward to produce token credit. Bars are schematic, and dashed frames group stages.}
\label{fig:hp_rlvr_overview}
\end{figure*}

Reinforcement learning with verifiable rewards (RLVR) trains language models through outcomes checked by symbolic verifiers, test suites, or interactive environments \citep{deepseekai2025r1}. Terminal verification supplies supervision without labeling every intermediate decision. Turning that supervision into efficient learning, however, requires addressing both the cost of obtaining an update signal and its allocation within a trajectory.

First, group-relative learning obtains its signal through repeated attempts. Group Relative Policy Optimization (GRPO) compares multiple responses to each prompt, replacing a learned critic with a baseline computed from group rewards \citep{shao2024deepseekmath}. This comparison requires multiple rollouts, yet homogeneous outcomes give zero reward advantages \citep{yu2025dapo}. Additional sampling therefore incurs cost without necessarily supplying an informative comparison.

Second, comparisons between complete attempts leave credit coarse within each attempt. Outcome-supervised GRPO broadcasts a trajectory's advantage to its tokens \citep{shao2024deepseekmath}, despite the different roles of intermediate decisions in reasoning and multi-turn interaction \citep{lightman2023verify,feng2025gigpo}. Process supervision supplies intermediate labels \citep{lightman2023verify}, implicit process rewards learn dense signals \citep{yuan2024free,cui2025prime}, and intermediate-value estimation can require additional continuations \citep{kazemnejad2025vineppo}. These approaches motivate a complementary question: can a single verified attempt, whether successful or failed, itself supply a finer learning signal?

A completed attempt's realized future provides context for revisiting earlier decisions \citep{harutyunyan2019hindsight,mesnard2021counterfactual}. We introduce Hindsight-Privileged RLVR (HP-RLVR), which uses this future as training-only privilege (\cref{fig:hp_rlvr_overview}). The generating policy compares token distributions with and without the future and verifier feedback. Historical reward centering directs the update and can give failures negative credit, while normalized Jensen--Shannon divergence supplies bounded, nonnegative token weights. This yields local credit from one rollout per prompt, without a reference solution or learned critic (\cref{sec:method}).

We evaluate Qwen3 models (1.7B/4B/8B) \citep{qwen3team2025qwen3} on mathematics, code generation, and ALFWorld \citep{shridhar2021alfworld}. \Cref{fig:single_rollout_cost} summarizes the Qwen3-8B cost benefit at a shared prompt and step budget, using one eighth as many rollouts per prompt. Task accuracy and training-signal diagnostics assess the resulting single-rollout updates (\cref{fig:checkpoint_accuracy,fig:training_signal}).

Our contributions are:
\begin{itemize}\setlength{\itemsep}{2pt}
\item A single-rollout hindsight construction using the generating policy and its own realized future, without reference solutions or a learned critic (\cref{sec:method_scoring}).
\item A bounded token-credit rule for both successful and failed trajectories, combining historical reward centering with normalized distributional divergence (\cref{sec:method_credit}).
\item Lower training cost at a shared step schedule, with accuracy comparisons and protocol diagnostics (\cref{sec:experiments}).
\end{itemize}

\section{Background and Related Work}
\label{sec:related}

\subsection{Related work}
\label{sec:related_work}

\paragraph{Outcome-based and hindsight credit assignment.}
Methods for assigning local credit differ in how they obtain information beyond a terminal reward. Learned value baselines estimate future returns \citep{schulman2016highdimensional}, and RUDDER redistributes delayed rewards through a return predictor \citep{arjona2019rudder}. In language-model reasoning, process supervision supplies intermediate labels \citep{lightman2023verify}, implicit process rewards learn dense signals from outcome supervision \citep{yuan2024free,cui2025prime}, and VinePPO estimates intermediate values through additional continuations \citep{kazemnejad2025vineppo}. Hindsight credit assignment instead uses observed future information to revisit past actions \citep{harutyunyan2019hindsight}. Counterfactual credit assignment conditions value functions on suitable future events to reduce policy-gradient variance \citep{mesnard2021counterfactual}.

This distinction also matters for multi-turn agents. GiGPO augments episode-level group advantages with step-level comparisons at repeated states across sampled trajectories \citep{feng2025gigpo}. HCAPO uses the language model as a post-hoc critic and combines hindsight action evaluation with multi-scale advantages \citep{tan2026hcapo}. HP-RLVR draws on the same idea of revisiting decisions after completion, but constructs its weights from the full-vocabulary divergence between two views of the generating snapshot. 

\paragraph{Privileged supervision and failures.}
On-policy distillation scores student-generated sequences with a teacher \citep{agarwal2024onpolicy}. OPSD conditions the same model on privileged reference information \citep{zhao2026selfdistilled}. Feedback-based self-distillation obtains privilege from interaction \citep{hubotter2026reinforcement,liu2026hero,yeo2026hintsd}. H$^2$SD uses verified student responses to refine successful trajectories and reference hints to correct failures \citep{cai2026h2sd}. HP-RLVR instead uses an unsigned divergence to weight reward-based updates, including updates on failures without a corrective reference. Negative reinforcement can itself improve reasoning \citep{zhu2025negative}, but update dynamics depend on weighting \citep{chen2026disentangled}, loss normalization, and length dynamics \citep{liu2025understandingr1zero,yu2025dapo}. HP-RLVR and GRPO use trajectory-balanced loss reduction, and we report training-signal availability under this shared protocol (\cref{sec:exp_diag}).

\subsection{Problem formulation}
\label{sec:prelim_rlvr}

Let $\pi_\vtheta$ generate a response or interactive trajectory $\vy$ for a prompt $\vx\sim\mathcal D$. A verifier assigns $R(\vx,\vy)\in[0,1]$, binary for task success or graded for partial credit. The objective is
\begin{equation}
\mathcal J_{\mathrm{RLVR}}(\theta)=\E_{\vx\sim\mathcal D,\,\vy\sim\pi_\vtheta}[R(\vx,\vy)].
\label{eq:rlvr}
\end{equation}

For interactive trajectories, this expectation also includes environment transitions, but only policy-generated tokens contribute policy scores. With a parameter-independent verifier and environment, the score-function gradient is $\E[R\nabla_\vtheta\log\pi_\vtheta(\vy\mid\vx)]$.

GRPO samples $G$ trajectories per prompt and assigns
\begin{equation}
\hat A_i=\frac{R_i-\overline R}{\sigma_R+\epsilon_{\mathrm{stab}}},
\qquad
\overline R=\frac1G\sum_{i=1}^G R_i,
\label{eq:grpo_adv}
\end{equation}
where $\sigma_R$ is the group reward standard deviation and $\epsilon_{\mathrm{stab}}>0$ stabilizes the denominator. The same advantage applies to every generated token of trajectory $i$. Homogeneous groups give $\hat A_i=0$ for all members. \Appref{app:notation} summarizes notation.

\section{Hindsight-Privileged RLVR}
\label{sec:method}

\subsection{From trajectory rewards to token credit}
\label{sec:method_setup}

With a single rollout, group-relative credit vanishes: $\overline R=R_i$ and $\hat A_i=0$ (\cref{eq:grpo_adv}). A baseline fixed before sampling can instead supply a reward comparison without requiring another attempt at the same prompt. The following standard identity provides the basis for this construction.

\begin{samepage}
\begin{lemma}[Historical baseline invariance]
\label{lem:baseline}
Condition on the completed history $\mathcal H_k$, and let $\E_k$ denote expectation over a fresh prompt from $\mathcal D$ and a matched on-policy rollout. Assume the task distribution, verifier, and environment are parameter-independent, and the trajectory law has differentiable, fixed support with integrably dominated derivatives as specified in \appref{app:baseline_assumptions}. Let $g_\theta$ be the sum of policy scores over all generated tokens. Any $\mathcal H_k$-measurable baseline $b_k\in[0,1]$, held fixed when differentiating, satisfies
\begin{equation}
\E_k[(R-b_k)g_\theta]
=\nabla_\vtheta\mathcal J_{\mathrm{RLVR}}(\theta).
\label{eq:historical_baseline_identity}
\end{equation}
\end{lemma}
\end{samepage}

Thus a single on-policy trajectory yields an unbiased estimator of the unweighted reward gradient under these assumptions. We choose the mean reward of earlier trajectories. For batch $k$, let $\mathcal C_k$ contain trajectories completed before that batch, and define
\begin{equation}
\begin{aligned}
b_k&=\frac1{|\mathcal C_k|}\sum_{i'\in\mathcal C_k}R_{i'},\\
s_i&=R_i-b_k.
\end{aligned}
\label{eq:baseline}
\end{equation}
When history is empty, use an initial $b_{\mathrm{init}}\in[0,1]$ fixed before sampling. The current batch enters the running sum and count only after its advantages are constructed.

Broadcasting $s_i$ still assigns the same credit to every token. For generated position $j$, we therefore separate trajectory evaluation from local allocation:
\begin{equation}
A_j=\underbrace{(R_i-b_k)}_{\text{trajectory evaluation}}\,
\underbrace{w_j}_{\text{local allocation}}.
\label{eq:advantage}
\end{equation}
Here, $w_j\in[0,1]$ is a token weight. The completed trajectory's future offers additional context for revisiting earlier decisions. We use its effect on the generating policy's predictions to construct $w_j$ (\cref{fig:hp_rlvr_overview}).

\subsection{Hindsight-conditioned scoring}
\label{sec:method_scoring}

Let $h_j$ be the history visible when token $y_j$ was generated, $F_j$ the allowed future from the same trajectory, and $O$ the textual verifier report. With the generating snapshot frozen, the two scoring distributions are
\begin{equation}
\begin{aligned}
p_j&=\pi_{\theta_{\mathrm{old}}}(\cdot\mid h_j),\\
q_j&=\pi_{\theta_{\mathrm{old}}}(\cdot\mid h_j,F_j,O).
\end{aligned}
\label{eq:views}
\end{equation}
Both views score the same position. $F_j$ excludes $y_j$ and contains the allowed continuation or later agent transitions. Only the training scorer receives $F_j$ and $O$ (\appref{app:sampling_scoring}). Both views reuse the same completed trajectory, so no additional continuations are sampled (\cref{alg:hp_rlvr}). Their distributional change supplies the local weight defined below.

A task adapter renders the verifier report as an outcome descriptor $O$, for example ``the completed trajectory was unsuccessful'' or ``the solution passed $7$ of $10$ tests.'' The adapter restates the report without adding an independent correctness judgment. An eligible future $F_j$ contains only allowed information from the same realized trajectory and excludes the scored token itself. For a single response it is the allowed continuation after that token; for an agent it comprises the subsequent executed actions and observations, starting at the next transition. Scoring turn $t$ therefore uses transitions from $(\mathrm{action}_{t+1},\mathrm{observation}_{t+2})$ onward, excluding the immediate observation that could trivially reveal the current action. Future information is available only after completion, and never during rollout generation or deployment.

\subsection{Outcome-centered token credit}
\label{sec:method_credit}

The local allocation should vanish when the causal and hindsight views agree, and be positive when they differ. Jensen--Shannon divergence (JSD) satisfies this criterion \citep{lin1991divergence}.

\begin{lemma}[Zero divergence]
\label{lem:null_hindsight}
Let $p,q$ be probability mass functions on a finite vocabulary $\mathcal V$, and let $m=(p+q)/2$. With natural logarithms and zero-mass terms interpreted as zero,
\begin{equation*}
\operatorname{JSD}(p,q)
=\tfrac12 D_{\mathrm{KL}}(p\Vert m)
+\tfrac12 D_{\mathrm{KL}}(q\Vert m)
\end{equation*}
is nonnegative and equals zero if and only if $p(v)=q(v)$ for every $v\in\mathcal V$.
\end{lemma}

Motivated by \cref{lem:null_hindsight}, we use the normalized divergence as the token weight. With $m_j=(p_j+q_j)/2$, define
\begin{equation}
\begin{aligned}
w_j&=\frac{\operatorname{JSD}(p_j,q_j)}{\log 2}
\in[0,1],
\end{aligned}
\label{eq:weight}
\end{equation}
The normalization gives the following bound on the resulting credit.

\begin{proposition}[Bounded credit]
\label{prop:bounded_credit}
Let $p_j,q_j$ be probability mass functions on $\mathcal V$, $R_i,b_k\in[0,1]$, and define $w_j,A_j$ by \cref{eq:weight,eq:advantage}. Every scored position satisfies
\begin{equation*}
0\le w_j\le1,\qquad |A_j|\le|R_i-b_k|\le1.
\end{equation*}
Moreover, $w_j=1$ if and only if $p_j$ and $q_j$ have disjoint supports.
\end{proposition}

The advantages are detached before optimization, and positions with no eligible future are masked out. For scored positions, \cref{lem:null_hindsight} gives $A_j=0$ exactly when $R_i=b_k$ or $p_j=q_j$. Otherwise, the centered reward determines the sign and the divergence determines the relative magnitude within a trajectory. A zero-advantage position contributes zero to the credit-weighted surrogate gradient. Although reward enters explicitly as one scalar factor, $w_j$ also depends on its rendered descriptor $O$.

\Cref{lem:baseline} applies to the unweighted trajectory score, before hindsight weighting, token-length normalization, and clipping. Since $w_j$ depends on the sampled trajectory, its baseline contribution need not have zero expectation. HP-RLVR therefore optimizes a weighted surrogate. Proofs of the two lemmas and the proposition appear in \appref{app:credit_proofs}.

\subsection{Training objective and reduction}
\label{sec:method_objective}

Let $\mathcal S_i$ be the eligible generated positions of trajectory $i$, $N_i=|\mathcal S_i|$, and $B$ the number of contributing trajectories. We maximize the empirical PPO-style surrogate \citep{schulman2017ppo}
\begin{equation}
\begin{aligned}
\widehat{\mathcal J}_{\mathrm{HP}}(\theta)
&=\frac1B\sum_{i=1}^B\frac1{N_i}\sum_{j\in\mathcal S_i}\ell_j(\theta),\\
\ell_j(\theta)
&=\min\bigl(r_jA_j,\operatorname{clip}(r_j,1-\epsilon,1+\epsilon)A_j\bigr),\\
r_j(\theta)&=\frac{\pi_\vtheta(y_j\mid h_j)}{\pi_{\theta_{\mathrm{old}}}(y_j\mid h_j)}.
\end{aligned}
\label{eq:objective}
\end{equation}

Only the current policy receives gradients, with prompts, observations, and hindsight context excluded from the policy loss. We use one optimization epoch and $\epsilon=0.2$.

Both methods average over valid tokens within each trajectory and then over trajectories. For multi-turn data stored as rows, a row with $N_{it}$ valid tokens receives weight $N_{it}/N_i$ within its trajectory. Successful episodes can terminate before failed episodes, so equal turn-row weights can give longer episodes greater aggregate weight. This reduction avoids that extra weighting by turn count. Implementation checks are described in \appref{app:implementation}.

\Cref{alg:hp_rlvr} summarizes one rollout/training step. Sampling and scoring follow \cref{sec:method_setup,sec:method_scoring}. Both scoring views are detached, divergence is evaluated over the full vocabulary, and only eligible generated tokens enter the trajectory-balanced loss.

\begin{algorithm}[!t]
\caption{HP-RLVR training step}
\label{alg:hp_rlvr}
\begin{algorithmic}[1]
\renewcommand{\algorithmiccomment}[1]{\hfill$\triangleright$ \textit{#1}}
\REQUIRE policy $\pi_\vtheta$, prompt batch $\{\vx_i\}$, verifier $R$, completed history $\mathcal C_k$
\ENSURE  updated policy $\pi_\vtheta$, extended history $\mathcal C_{k+1}$
\STATE $\theta_{\mathrm{old}} \gets \theta$ \COMMENT{frozen generating snapshot}
\STATE $b_k \gets \mathrm{mean}\{R_{i'} : i' \in \mathcal C_k\}$ \COMMENT{\cref{eq:baseline}}
\STATE $\vy_i \sim \pi_{\theta_{\mathrm{old}}}(\cdot \mid \vx_i)$ \COMMENT{one rollout per prompt}
\STATE Record $\log \pi_{\theta_{\mathrm{old}}}(y_j \mid h_j)$ 
\STATE $R_i \gets R(\vx_i, \vy_i)$ \COMMENT{verifier reward}
\STATE $O_i \gets \mathrm{Render}(R_i)$ \COMMENT{outcome descriptor}
\STATE $p_j \gets \pi_{\theta_{\mathrm{old}}}(\cdot \mid h_j)$
\STATE $q_j \gets \pi_{\theta_{\mathrm{old}}}(\cdot \mid h_j, F_j, O_i)$ \COMMENT{\cref{eq:views}}
\STATE $w_j \gets \operatorname{JSD}(p_j, q_j) / \log 2$ \COMMENT{\cref{eq:weight}}
\STATE $A_j \gets (R_i - b_k)\, w_j$ \COMMENT{dvantage (\cref{eq:advantage})}
\STATE Update $\vtheta$  \COMMENT{\cref{eq:objective}}
\STATE $\mathcal C_{k+1} \gets \mathcal C_k \cup \{R_i\}$ \COMMENT{history updated after credit}
\end{algorithmic}
\end{algorithm}

\section{Experiments}
\label{sec:experiments}

Our evaluation tests whether training-only hindsight can replace group sampling at lower computational cost, and how that substitution affects task accuracy and the availability of a training signal. We study sequence-level verification in mathematics and code generation, and sparse terminal rewards in multi-turn ALFWorld interaction.

\subsection{Experimental setup}
\label{sec:exp_setup}

We evaluate Qwen3-1.7B, 4B, and 8B \citep{qwen3team2025qwen3}, comparing HP-RLVR ($G{=}1$) with a critic-free GRPO control ($G{=}8$). Paired runs share task data, seed $42$, LoRA configuration \citep{hu2021lora}, Adam settings \citep{kingma2014adam}, and one PPO epoch per rollout step. Both use full-support training sampling and trajectory-balanced loss reduction. We evaluate at steps $25$, $50$, $75$, and $100$. A training step comprises one rollout-and-training iteration, with run-specific trajectory budgets and internal optimizer-step counts. Run-specific batching records appear in \appref{app:setup_details}.

Both methods use LoRA rank $64$, $\alpha=128$, and dropout $0$, applied to attention and MLP projections, with Adam betas $0.9$ and $0.999$, zero weight decay, a constant learning rate of $5\times10^{-6}$, gradient clipping $1.0$, and a PPO clip range of $0.2$. Each step takes one optimization epoch over its rollout batch, and the shared schedule runs $100$ rollout-and-training steps with checkpoints every $25$ steps. Training steps count rollout-and-training iterations, separately from internal optimizer steps and log records.

Mathematics trains on a $30$k-prompt OpenThoughts mixture \citep{guha2025openthoughts}, and code generation trains on $800$ problems divided into $60\%$ standard-I/O and $40\%$ function-call tasks. Both cap training responses at $1{,}024$ tokens. ALFWorld uses four prompts per rollout step, $50$ turns, an eight-turn retained history, a $4{,}096$-token per-turn response cap, and an $8{,}192$-token training context. Mathematics uses an official symbolic verifier, code generation a test harness, and ALFWorld the environment's terminal success signal.

Both methods sample at temperature $1$ without top-$p$ or top-$k$ truncation, which keeps the sampled distribution, the behavior scoring, and the likelihood-ratio denominator aligned. At each generated position $j$ the rollout engine records the token $y_j$ and its generation-time log-probability, which supplies the behavior denominator of \cref{eq:objective}; for agents, $h_j$ contains the retained observation--action history rather than necessarily the entire episode. The full behavior distribution of \cref{eq:views} is obtained by replaying the same history with the generating snapshot, and a check at the first update replays current log-probabilities against that record, aborting the run on a mismatch rather than biasing the ratios. The exact divergence of \cref{eq:weight} is evaluated over the full vocabulary in groups of token positions, with no top-$k$ or restricted-support approximation: stored token weights require $O(T)$ memory and temporary divergence calculations $O(CV)$ for $T$ positions, vocabulary size $V$, and groups of at most $C$ positions, excluding model states and activations.

Mathematics uses AIME24, AIME25, and HMMT25 \citep{balunovic2025matharena,hmmt2025}, with $30$ problems per benchmark and $12$ samples per problem, thinking enabled, temperature $1$, top-$p$ $0.95$, and a $38{,}912$-token response cap. Avg@$12$ averages sampled correctness and Pass@$12$ measures whether any sample solves a problem. Mathematics aggregates average the three benchmarks. Code evaluates at temperature $0.2$ on all $1{,}055$ LiveCodeBench release-$6$ problems \citep{jain2024livecodebench}, with Pass@$1$ and Pass@$8$ computed as eight-sample estimates from a common dump. ALFWorld reports the mean success rate across in-distribution (ID, $140$ games) and out-of-distribution (OOD, $134$ games) splits \citep{shridhar2021alfworld}, evaluating one episode per game under the $50$-turn protocol. Each method has one training seed. All timings use 8 Ascend 910B NPUs.

\paragraph{Implementation Details}
\label{app:implementation}

The dense packed attention path gives approximately $5\times10^{-4}$ to $6\times10^{-4}$ nats of KL against a reference implementation on identical context. Results from the discarded attention prototype are excluded.

For the released implementation, the trajectory-balanced reduction of \cref{sec:method_objective} carries per-trajectory assertions: the within-trajectory row weights $w_{it}=N_{it}/N_i$ must sum to one, and the effective trajectory weight must be uncorrelated with its turn and token counts. The clipped surrogate of \cref{eq:objective} is applied with the coefficient used by the GRPO control in the same stack.

\subsection{Cost of replacing group sampling}
\label{sec:exp_efficiency}

The practical benefit of $G{=}1$ depends on whether its additional hindsight scoring costs less than the group rollouts it replaces. We test this on Qwen3-8B at a shared four-prompt budget and $100$ rollout steps, generating four trajectories per step for HP-RLVR and $32$ for GRPO. \Cref{fig:single_rollout_cost} shows lower end-to-end training time in all three task families, with GRPO-to-HP-RLVR ratios of $1.21$--$2.30$. The largest reduction occurs on ALFWorld, saving $255.8$ NPU-hours despite the additional scoring passes. Wall-clock spans each run's first and last timestamps. Absolute timings and NPU-hours appear in \cref{tab:efficiency_scales}. The comparison measures cost at a fixed prompt and step budget, rather than at equal trajectory counts or matched accuracy.

\begin{table*}[t]
\centering
\small
\setlength{\tabcolsep}{4pt}
\caption{Training-cost records for completed $100$-step runs. NPU-hours are separately logged run totals. Ratio: GRPO/HP-RLVR training time, so values below one favor GRPO. The 4B mathematics pair uses $32$/$16$ prompts per step for HP-RLVR/GRPO.}
\label{tab:efficiency_scales}
\begin{tabular}{@{}ll rr rr r@{}}
\toprule
& & \multicolumn{2}{c}{Time (min)} & \multicolumn{2}{c}{NPU-hours} & \\
\cmidrule(lr){3-4}\cmidrule(lr){5-6}
Model & Task & GRPO & HP-RLVR & GRPO & HP-RLVR & Ratio \\
\midrule
1.7B & Mathematics & $92.70$ & $140.27$ & $12.36$ & $18.70$ & $0.66\times$ \\
 & Code & $67.55$ & $\mathbf{45.80}$ & $9.01$ & $\mathbf{6.11}$ & $\mathbf{1.48\times}$ \\
 & ALFWorld & $3242.85$ & $\mathbf{1003.25}$ & $432.4$ & $\mathbf{133.8}$ & $\mathbf{3.23\times}$ \\
\midrule
4B & Mathematics & $175.66$ & $\mathbf{125.47}$ & $23.42$ & $\mathbf{16.73}$ & $\mathbf{1.40\times}$ \\
 & Code & $88.80$ & $\mathbf{61.09}$ & $11.84$ & $\mathbf{8.15}$ & $\mathbf{1.45\times}$ \\
 & ALFWorld & $3067.50$ & $\mathbf{981.00}$ & $409$ & $\mathbf{130.8}$ & $\mathbf{3.13\times}$ \\
\midrule
8B & Mathematics & $167.68$ & $\mathbf{138.30}$ & $22.4$ & $\mathbf{18.4}$ & $\mathbf{1.21\times}$ \\
 & Code & $98.77$ & $\mathbf{67.87}$ & $13.2$ & $\mathbf{9.0}$ & $\mathbf{1.46\times}$ \\
 & ALFWorld & $3389.55$ & $\mathbf{1471.05}$ & $451.9$ & $\mathbf{196.1}$ & $\mathbf{2.30\times}$ \\
\bottomrule
\end{tabular}
\end{table*}

\paragraph{Where the control spends computation.}
\Cref{tab:grpo_phase_cost} decomposes the recorded 4B GRPO mathematics run. Generation accounts for $72.9\%$ of its wall-clock, compared with $12.7\%$ for policy training and $5.5\%$ for behavior-policy replay. This run generates $12{,}800$ trajectories from $16$ prompts per step, a different prompt budget from the 8B pairs. Its $1{,}024$-token training cap truncates $41.6\%$ of responses. The profile identifies generation as the largest measured cost in this control. It does not provide a paired phase decomposition for HP-RLVR or isolate the cost of hindsight scoring.

\begin{table}[!t]
\centering
\small
\setlength{\tabcolsep}{5pt}
\caption{Qwen3-4B GRPO mathematics timing over $100$ steps at $16$ prompts and $128$ trajectories per step. Residual time is the total minus the four listed phase measurements.}
\label{tab:grpo_phase_cost}
\begin{tabular}{@{}lrr@{}}
\toprule
Phase & Time (s) & Share (\%) \\
\midrule
Rollout generation & $7{,}678.48$ & $72.9$ \\
Policy training & $1{,}343.70$ & $12.7$ \\
Behavior-policy replay & $576.85$ & $5.5$ \\
Weight synchronization & $141.96$ & $1.3$ \\
Residual time & $798.51$ & $7.6$ \\
\midrule
Total & $10{,}539.50$ & $100.0$ \\
\bottomrule
\end{tabular}
\end{table}

\subsection{Task accuracy under single-rollout training}
\label{sec:exp_benchmarks}

To assess the performance--time tradeoff, \cref{fig:checkpoint_accuracy} pairs each configuration's best observed task score within $100$ steps with its measured complete-run training time. Checkpoint selection maximizes the mean benchmark score at a single checkpoint. Mathematics averages AIME24, AIME25, and HMMT25 Avg@$12$, code uses eight-sample mean Pass@$1$ on LiveCodeBench, and ALFWorld averages ID and OOD success. \Cref{tab:benchmarks} reports each configuration's best observed value per metric. Both methods are evaluated at all four checkpoints at every size, and code passes use eight completions per problem from a common dump.

\begin{table*}[t]
\centering
\small
\setlength{\tabcolsep}{3.5pt}
\caption{Qwen3-1.7B, 4B, and 8B results under the current protocol (\%). Each cell is the best observed value for that metric across the four checkpoints. Avg@$12$ averages correctness over $12$ samples per problem and the mathematics \emph{Avg.} is the mean over the three benchmarks. Code reports Pass@$1$ and Pass@$8$ on $1{,}055$ LiveCodeBench problems. ALFWorld reports ID/OOD success ($140$/$134$ games) and their mean.}
\label{tab:benchmarks}
\begin{tabular}{@{}l l cccc c cc ccc@{}}
\toprule
& & \multicolumn{4}{c}{Mathematics (Avg@$12$)} & & \multicolumn{2}{c}{Code} & \multicolumn{3}{c}{ALFWorld} \\
\cmidrule(lr){3-6}\cmidrule(lr){8-9}\cmidrule(lr){10-12}
Size & Method & AIME24 & AIME25 & HMMT25 & Avg. & Pass@$12$ & Pass@$1$ & Pass@$8$ & ID & OOD & Avg. \\
\midrule
1.7B & GRPO ($G{=}8$) & $51.90$ & $38.10$ & $25.00$ & $37.80$ & $65.56$ & $25.78$ & $32.04$ & $40.00$ & $39.55$ & $38.71$ \\
1.7B & HP-RLVR ($G{=}1$) & $52.20$ & $40.30$ & $27.20$ & $38.90$ & $66.67$ & $25.11$ & $31.75$ & $32.14$ & $31.34$ & $30.64$ \\
\midrule
4B & GRPO ($G{=}8$) & $71.39$ & $64.44$ & $46.39$ & $60.09$ & $76.67$ & $37.29$ & $45.31$ & $70.00$ & $66.42$ & $68.21$ \\
4B & HP-RLVR ($G{=}1$) & $76.67$ & $67.50$ & $46.39$ & $62.32$ & $78.89$ & $37.24$ & $44.55$ & $62.14$ & $63.43$ & $61.72$ \\
\midrule
8B & GRPO ($G{=}8$) & $74.20$ & $72.20$ & $45.00$ & $63.07$ & $81.10$ & $40.65$ & $49.29$ & $65.71$ & $64.93$ & $65.32$ \\
8B & HP-RLVR ($G{=}1$) & $75.00$ & $73.30$ & $49.20$ & $65.73$ & $79.97$ & $38.87$ & $47.30$ & $62.14$ & $64.18$ & $62.45$ \\
\bottomrule
\end{tabular}
\end{table*}

\begin{figure*}[!t]
\centering
\includegraphics[width=\linewidth]{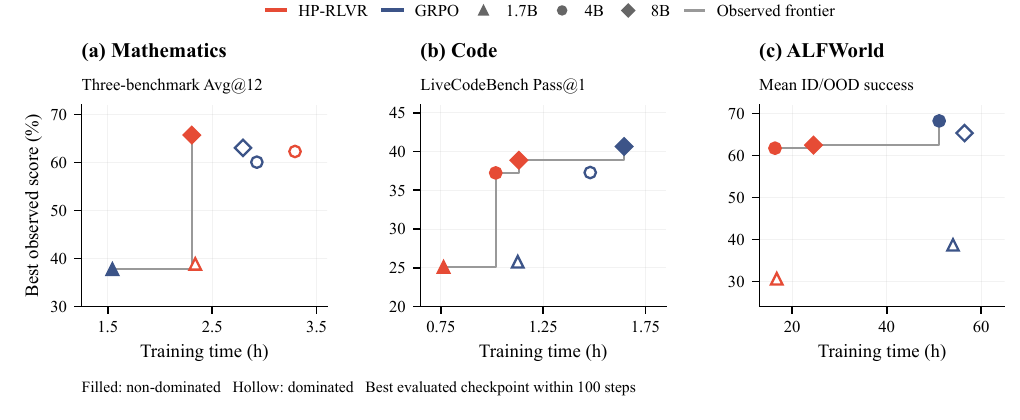}
\caption{Best observed task performance versus complete-run training time over $100$ steps. Each method and model size uses the highest whole-task score at steps $25/50/75/100$. Filled and hollow points are non-dominated and dominated within each task, respectively, and gray steps trace the observed budget envelope. Code uses Pass@$1$ over $1{,}055$ problems. Single-seed runs differ in hardware and prompt budgets, including $32$/$16$ prompts per step for 4B mathematics (HP-RLVR/GRPO). The 1.7B HP-RLVR ALFWorld timing includes a restart.}
\label{fig:checkpoint_accuracy}
\end{figure*}

At 8B, HP-RLVR's best mathematics score is $65.73\%$ versus $63.07\%$ for GRPO, with $1.21\times$ shorter training time. Code and ALFWorld expose a different tradeoff: HP-RLVR trains $1.46\times$ and $2.30\times$ faster, respectively, with best observed scores of $38.87\%$ versus $40.65\%$ and $62.45\%$ versus $65.32\%$. HP-RLVR contributes non-dominated configurations in all three task families. The frontier describes the measured configurations and does not establish statistical superiority or isolate estimator effects across hardware and budgets.

In mathematics, HP-RLVR has the higher best observed Avg@$12$ at all three scales, by $1.1$ to $2.7$ points: $38.90\%$ against $37.80\%$ at 1.7B, $62.32\%$ against $60.09\%$ at 4B, and $65.73\%$ against $63.07\%$ at 8B. It also has the higher Pass@$12$ at 1.7B and 4B, extending the comparison to solution coverage, while GRPO reaches the higher 8B Pass@$12$ ($81.10\%$ versus $79.97\%$). These results show stronger accuracy retention relative to the control under a constrained training regime: responses are capped at $1{,}024$ tokens during training and $38{,}912$ during evaluation, and approximately $42\%$ of the 4B control's training responses reach the cap.

The ordering changes for code generation, where GRPO has the higher best observed Pass@$1$ and Pass@$8$ at all three sizes. The margin is narrowest at 4B, where HP-RLVR reaches $37.24\%$ Pass@$1$ and $44.55\%$ Pass@$8$ against $37.29\%$ and $45.31\%$ for GRPO on $1{,}055$ problems. On ALFWorld, GRPO has the higher ID success, OOD success, and mean at all three sizes. Together, these comparisons indicate a task-dependent accuracy tradeoff when replacing group-relative updates with single-rollout credit.

The 8B runs evaluate all three task families (\cref{tab:benchmarks}). Code generation and ALFWorld both favor GRPO at 8B, and mathematics favors HP-RLVR, at $65.73\%$ against $63.07\%$.

\paragraph{Best observed values.}
The benchmark-level results in \cref{tab:benchmarks} report each configuration's best observed value per metric. At 4B, GRPO has the higher ALFWorld ID success ($70.00\%$ versus $62.14\%$), OOD success ($66.42\%$ versus $63.43\%$), and mean ($68.21\%$ versus $61.72\%$). For the ID and OOD evaluations of $140$ and $134$ games, individual success-rate estimates near $65\%$ have approximate binomial standard errors of $4.1$ percentage points. Paired uncertainty requires the per-game discordance counts.

\subsection{Training-signal availability and controls}
\label{sec:exp_diag}
\label{sec:exp_failure_modes}
\label{sec:exp_component}

\begin{figure}[!t]
\centering
\includegraphics[width=\linewidth]{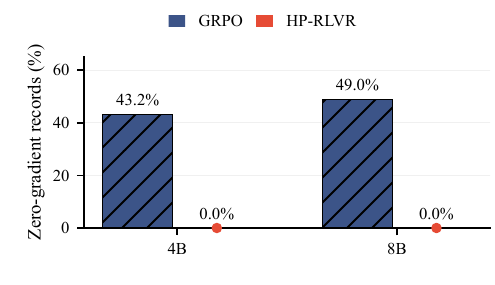}
\caption{ALFWorld zero gradient-norm records as a fraction of all logged records.}
\label{fig:training_signal}
\end{figure}

Group-relative credit vanishes when all sampled outcomes agree. HP-RLVR instead compares each outcome with a historical baseline, so failed trajectories can carry negative credit without requiring a successful attempt at the same prompt. The nonnegative JSD weights distribute this credit across eligible tokens (\cref{eq:advantage,eq:weight}). In 4B mathematics, $430$ of $902$ recorded reward groups ($47.7\%$) are homogeneous, supplying no reward-based advantage. In ALFWorld, zero gradient-norm records account for $43.2\%$ and $49.0\%$ of GRPO records at 4B and 8B, respectively, versus none for HP-RLVR (\cref{fig:training_signal}). These diagnostics support the availability of single-rollout updates without an informative within-prompt reward comparison. Counts appear on the plot and in \cref{tab:zerograd}. Reward groups measure uninformative comparisons, whereas logged gradient records measure current-gradient availability.

Both methods share trajectory-balanced reduction and interaction settings. Generation-time probabilities and replay checks align PPO ratios (\appref{app:sampling_scoring}). The comparisons assess the combined estimator: a matched $w_j=1$ control would test token weighting, and future-context controls would test the information source.

\paragraph{Evaluation coverage.}
The main-text \cref{tab:benchmarks} reports each configuration's best observed value. Both methods appear at all three sizes for mathematics, code, and ALFWorld. Absolute timings appear in \cref{tab:efficiency_scales}, and paired 1.7B comparisons appear in \appref{app:frozen_1p7b}.

\paragraph{Training-side accounting.}
The mathematics comparison combines a short training cap with long evaluation responses. In the Qwen3-4B GRPO run, $41.6\%$ of training responses reach the $1{,}024$-token cap, whereas evaluation allows $38{,}912$ tokens. The training statistics therefore include a substantial truncated-response population. Separately, $430$ of $902$ reward groups ($47.7\%$) have zero reward variance, giving zero group-relative advantage under \cref{eq:grpo_adv}. This demonstrates the absence of reward-based gradient contributions from those homogeneous groups, alongside the agent gradient-log diagnostic in \cref{tab:zerograd}.

\FloatBarrier
\section{Conclusion}
\label{sec:conclusion}

HP-RLVR assigns token credit from one rollout per prompt: historical reward centering sets the update direction, while bounded, nonnegative JSD weights from causal and hindsight scoring allocate its magnitude. This rule applies to successful and failed trajectories and can give failures negative credit without a corrective reference. At a shared prompt and $100$-step budget, paired Qwen3-8B runs train $1.21\times$, $1.46\times$, and $2.30\times$ faster than $G{=}8$ GRPO on mathematics, code, and ALFWorld, using one eighth as many rollouts per prompt. Accuracy is comparable overall, with HP-RLVR ahead on mathematics at all three sizes and GRPO ahead on code and ALFWorld. In ALFWorld, $43.2\%$ (4B) and $49.0\%$ (8B) of GRPO's gradient-norm records are zero, versus none for HP-RLVR.

\section*{Limitations}

Experiments cover Qwen3 models and text tasks with terminal verification. Other model families, multimodal feedback, and longer training contexts offer natural extensions. 

\bibliography{refs}

\appendix
\raggedbottom
\renewcommand{\theHproposition}{appendix.\thesection.\arabic{proposition}}
\renewcommand{\theHlemma}{appendix.\thesection.\arabic{lemma}}

\appendixcontentstrue
\begingroup
\renewcommand{\contentsname}{Appendix Contents}
\setcounter{tocdepth}{2}
\hypersetup{linkcolor=black}
\tableofcontents
\endgroup
\clearpage

\section{Notation and Preliminaries}
\label{app:notation}

\Cref{tab:notation} summarizes the notation for HP-RLVR and GRPO. Symbols used only in a proof are defined where introduced.

\begin{table}[!ht]
\centering
\small
\setlength{\tabcolsep}{4pt}
\caption{Core notation.}
\label{tab:notation}
\begin{tabularx}{\linewidth}{@{}l>{\raggedright\arraybackslash}X@{}}
\toprule
Symbol & Meaning \\
\midrule
$\vx,\vy$ & Prompt and generated trajectory. \\
$\pi_\vtheta,\pi_{\theta_{\mathrm{old}}}$ & Current policy and rollout snapshot. \\
$i,j,k$ & Trajectory, token, and batch indices. \\
$R_i,O$ & Verifier reward and outcome descriptor. \\
$G$ & Group size: rollouts per prompt. \\
$h_j,y_j$ & History and token at position $j$. \\
$F_j$ & Eligible future trace. \\
$p_j,q_j$ & Causal and hindsight snapshot distributions. \\
$\mathcal C_k$ & Rewards of trajectories completed before batch $k$. \\
$b_k$ & Mean reward of completed trajectories before batch $k$. \\
$w_j$ & Normalized JSD token weight. \\
$A_j$ & Token advantage $(R_i-b_k)w_j$. \\
$\hat A_i$ & GRPO group-relative advantage. \\
$r_j(\vtheta),\epsilon$ & Policy likelihood ratio and PPO clipping range. \\
\bottomrule
\end{tabularx}
\end{table}

\section{Proofs}
\label{app:proofs}
\label{app:credit_proofs}

We first establish the historical-baseline identity underlying single-rollout reward centering, then the divergence properties underlying local allocation. The former is an expectation identity for the unweighted trajectory score. The latter hold for each realized pair of scoring distributions.

\paragraph{Conditional sampling and regularity.}
\label{app:baseline_assumptions}
Fix a realization of the completed history $\mathcal H_k$ and a parameter value $\vtheta$ in an open neighborhood $U\subset\mathbb R^d$. All history-dependent quantities, including $b_k$, are held fixed when differentiating. The following assumptions specify the scope of \cref{lem:baseline}.
\begin{enumerate}
\item[(B1)] \emph{Matched conditional sampling.} Given this history, the current prompt has the parameter-independent law $\mathcal D$. Conditional on a prompt $\vx$, the current trajectory has density $f_\vtheta(\vy\mid\vx)$ with respect to a parameter-independent reference measure $\mu_\vx$. The law used to sample this trajectory is the same law being differentiated. The verifier $R(\vx,\vy)\in[0,1]$, environment transition kernels, and external stopping or truncation rules have no explicit parameter dependence. The rollout terminates almost surely after finitely many generated tokens. Its density factors into policy token probabilities and parameter-independent factors. Policy-generated stopping decisions are included among the scored tokens. For interactive tasks, $\vy$ includes the environment transitions.
\item[(B2)] \emph{Common support and differentiability.} For $\mathcal D$-almost every $\vx$, there is a measurable set $S_\vx$, independent of the parameter in $U$, on which $f_{\boldsymbol{\eta}}(\vy\mid\vx)>0$ for all $\boldsymbol{\eta}\in U$, up to a common $\mu_\vx$-null set. The density is zero off $S_\vx$, normalized for every $\boldsymbol{\eta}\in U$, and continuously differentiable in $\boldsymbol{\eta}$ on $S_\vx$.
\item[(B3)] \emph{Integrable domination.} There is a measurable $d(\vx,\vy)\ge0$ such that, almost everywhere,
\begin{equation*}
\sup_{\boldsymbol{\eta}\in U}\|\nabla_{\boldsymbol{\eta}} f_{\boldsymbol{\eta}}(\vy\mid\vx)\|
\le d(\vx,\vy),
\end{equation*}
and
\begin{equation*}
\int_{\mathcal X}\int d(\vx,\vy)\,\mu_\vx(d\vy)\,\mathcal D(d\vx)<\infty.
\end{equation*}
This permits differentiation under both the trajectory integral and the task expectation.
\end{enumerate}

By the autoregressive factorization in (B1), the trajectory score on its support is
\begin{equation}
\begin{aligned}
g_\theta(\vx,\vy)
&:=\nabla_\vtheta\log f_\vtheta(\vy\mid\vx)\\
&=\sum_{j\in\mathcal G(\vy)}
\nabla_\vtheta\log\pi_\vtheta(y_j\mid h_j),
\end{aligned}
\label{eq:trajectory_score}
\end{equation}
where $\mathcal G(\vy)$ contains \emph{all} policy-generated positions, including a sampled termination token when present. Observations contribute no policy score. This sum has no eligibility mask or token-length denominator. The policy probabilities must describe the actual sampling law, including its temperature convention. A parameter-dependent support change is excluded by (B2).

\paragraph{Divergence conventions.}
For probability mass functions $a,b$ on the finite vocabulary $\mathcal V$, define
\begin{equation*}
D_{\mathrm{KL}}(a\Vert b)
=\sum_{v:\,a(v)>0}a(v)\log\frac{a(v)}{b(v)},
\end{equation*}
with value $+\infty$ if $a(v)>0=b(v)$ for some $v$. Terms with $a(v)=0$ contribute zero, including when $b(v)=0$. For $m=(p+q)/2$, both $D_{\mathrm{KL}}(p\Vert m)$ and $D_{\mathrm{KL}}(q\Vert m)$ are finite because $m(v)\ge p(v)/2$ and $m(v)\ge q(v)/2$.

\subsection{Historical-baseline identity}
\label{app:proof_baseline}

\begin{proof}[Proof of \cref{lem:baseline}]
Work under the fixed history above and write $\E_k$ for expectation under (B1). Set the score to zero off $S_\vx$, a set of zero trajectory probability. Assumption (B3) gives
\begin{equation*}
\begin{aligned}
\E_k[\|g_\theta\|]
&=\int\!\!\int_{S_\vx}
\|\nabla_\vtheta f_\vtheta\|\,d\mu_\vx\,d\mathcal D\\
&\le\int\!\!\int d\,d\mu_\vx\,d\mathcal D<\infty.
\end{aligned}
\end{equation*}
Here and below the arguments $(\vy\mid\vx)$ are suppressed inside integrals. Since $R$ and $b_k$ are bounded, all score expectations used below are well-defined.

By (B2), $f_\vtheta g_\theta=\nabla_\vtheta f_\vtheta$ on the common support. Normalization and differentiation under the integral, justified by (B3), imply, for $\mathcal D$-almost every $\vx$,
\begin{equation*}
\begin{aligned}
\E_k[g_\theta\mid\vx]
&=\int_{S_\vx}\nabla_\vtheta f_\vtheta\,d\mu_\vx\\
&=\nabla_\vtheta\int_{S_\vx} f_\vtheta\,d\mu_\vx=0.
\end{aligned}
\end{equation*}
Consequently $\E_k[g_\theta]=0$. The same domination, together with $0\le R\le1$ and its parameter independence, gives
\begin{equation*}
\begin{aligned}
\nabla_\vtheta\mathcal J_{\mathrm{RLVR}}(\theta)
&=\nabla_\vtheta\int\!\!\int R f_\vtheta\,d\mu_\vx\,d\mathcal D\\
&=\int\!\!\int R\nabla_\vtheta f_\vtheta\,d\mu_\vx\,d\mathcal D\\
&=\E_k[R g_\theta].
\end{aligned}
\end{equation*}
The $\mathcal H_k$-measurable baseline is constant under this conditional expectation. Hence
\begin{equation*}
\begin{aligned}
\E_k[(R-b_k)g_\theta]
&=\E_k[Rg_\theta]-b_k\E_k[g_\theta]\\
&=\nabla_\vtheta\mathcal J_{\mathrm{RLVR}}(\theta),
\end{aligned}
\end{equation*}
which proves \cref{eq:historical_baseline_identity}. Applying the identity at the current policy parameter gives the conditional unbiasedness of a single sampled estimator. Earlier batches may be dependent, and $b_k$ may be correlated with the current policy through the history. The proof requires only the stated conditional sampling law and that $b_k$ is fixed before the current rollout. No variance-reduction or nonzero-gradient conclusion follows from this identity alone.
\end{proof}

\subsection{Zero divergence}
\label{app:proof_null_hindsight}

\begin{proof}[Proof of \cref{lem:null_hindsight}]
Let $\phi(t)=t\log t-t+1$ for $t>0$, with $\phi(0)=1$ by continuity. Its derivative is $\phi'(t)=\log t$, so $\phi$ decreases on $(0,1)$ and increases on $(1,\infty)$. Thus $\phi(t)\ge0$ for all $t\ge0$, with equality exactly at $t=1$.

Set $M=\{v\in\mathcal V:m(v)>0\}$. Both $p$ and $q$ vanish outside $M$. Using $\sum_{v\in M}p(v)=\sum_{v\in M}m(v)=1$,
\begin{equation*}
\begin{aligned}
D_{\mathrm{KL}}(p\Vert m)
&=\sum_{v\in M}m(v)\phi\!\left(\frac{p(v)}{m(v)}\right)\\
&\ge0.
\end{aligned}
\end{equation*}
Every coefficient $m(v)$ in this finite sum is positive. Equality therefore holds exactly when $p(v)/m(v)=1$ at every $v\in M$, equivalently $p=m$ on $\mathcal V$. The same argument applies to $q$. Since JSD is half the sum of these two nonnegative divergences, it is zero exactly when $p=m=q$. As $m=(p+q)/2$, this is equivalent to $p=q$. This proves both directions, including distributions with zero probabilities.
\end{proof}

\subsection{Bounded credit}
\label{app:proof_bounded_credit}

\begin{proof}[Proof of \cref{prop:bounded_credit}]
Fix a scored position and write $p=p_j$, $q=q_j$, and $m=(p+q)/2$. The lower bound on JSD follows from \cref{lem:null_hindsight}. For every $v$ with $p(v)>0$,
\begin{equation*}
\frac{p(v)}{m(v)}=\frac{2p(v)}{p(v)+q(v)}\le2.
\end{equation*}
Multiplying the logarithmic bound by $p(v)$ and summing gives
\begin{equation*}
D_{\mathrm{KL}}(p\Vert m)
\le\sum_{v:\,p(v)>0}p(v)\log2=\log2.
\end{equation*}
Equality holds exactly when $q(v)=0$ at every $v$ with $p(v)>0$, since every such summand has positive coefficient. The analogous bound holds for $q$, with equality exactly when $p$ vanishes on the support of $q$. Consequently
\begin{equation*}
0\le\operatorname{JSD}(p,q)\le\log2,
\end{equation*}
and the upper bound is attained exactly when the supports are disjoint. Dividing by $\log2>0$ proves the claims about $w_j$. Since $R_i,b_k\in[0,1]$,
\begin{equation*}
|A_j|=|R_i-b_k|w_j\le|R_i-b_k|\le1.
\end{equation*}
The historical mean lies in $[0,1]$ because it averages rewards in that interval, and the empty-history baseline lies there by assumption. The absolute bound $|A_j|=1$ is attained exactly when $(R_i,b_k)$ is $(1,0)$ or $(0,1)$ and the two supports are disjoint. In particular, overlapping supports imply $w_j<1$.
\end{proof}

\paragraph{Consequences for the surrogate.}
Combining \cref{lem:null_hindsight,prop:bounded_credit} with $A_j=(R_i-b_k)w_j$ gives, at every scored position,
\begin{equation*}
A_j=0\quad\Longleftrightarrow\quad
R_i=b_k\ \text{or}\ p_j=q_j.
\end{equation*}
Because $A_j$ is detached, $A_j=0$ makes the position's entire credit-weighted objective term identically zero wherever the policy ratio is defined. Its gradient is therefore zero, including at the clipping boundaries. The converse for gradients need not hold: clipping or a zero policy score can also yield a zero gradient at a position with $A_j\ne0$. The magnitude bound concerns advantages, not policy scores or optimizer stability.

For trajectory-dependent weights, the baseline term contains expectations of the form
\begin{equation*}
b_k\E_k\!\left[\sum_{j\in\mathcal S_i}
w_j\nabla_\vtheta\log\pi_\vtheta(y_j\mid h_j)\right],
\end{equation*}
which the score identity does not force to vanish. Detaching $w_j$ removes its derivative from the surrogate but does not remove its statistical dependence on the rollout. Eligibility masking, token-length normalization, and clipping further modify the update. Thus \cref{lem:baseline} justifies historical centering for the unweighted single-rollout estimator. It does not establish unbiasedness of the final HP-RLVR surrogate. The divergence results characterize predictive disagreement, not causal responsibility for the verified outcome.

\section{Method Details}
\label{app:method_details}

\subsection{Sampling and scoring details}
\label{app:sampling_scoring}

\paragraph{Sampling and behavior-policy consistency.}
For each prompt, $\pi_{\theta_{\mathrm{old}}}$ samples one trajectory ($G=1$) and receives its verifier reward:
\begin{equation}
\begin{aligned}
\vy_i&\sim\pi_{\theta_{\mathrm{old}}}(\cdot\mid\vx_i),\\
R_i&=R(\vx_i,\vy_i).
\end{aligned}
\label{eq:single_rollout}
\end{equation}
Generation-time recording, the temperature convention, the replay check, and the memory footprint of the exact divergence appear with the experimental setup in \cref{sec:exp_setup}. Evaluation uses each benchmark's specified sampling protocol.

\paragraph{Outcome descriptor and allowed future.}
The construction of $O$ and the eligibility rule for $F_j$ appear in \cref{sec:method_scoring}. Both scoring views use the same detached generating snapshot.

\section{Experimental Details}
\label{app:experimental_details}

\subsection{Setup Details}
\label{app:setup_details}

\paragraph{Batching and accounting records.}
The main text reports the optimizer, schedule, data, and evaluation settings shared by both methods. What follows is the run-specific batching record. The Qwen3-4B GRPO mathematics notes record $16$ prompts per step and $128$ generated sequences, with the global batch configured to $128$. This differs from the four-prompt accounting of the Qwen3-8B timing runs. ALFWorld uses a global batch setting of $32$ turn rows. A global turn-row batch is not an environment-trajectory budget, and per-step trajectory counts differ between the timing runs and the longer evaluation runs.

\subsection{Paired comparisons at Qwen3-1.7B}
\label{app:frozen_1p7b}

These comparisons assess whether the 4B mathematics ordering extends to a smaller model, and whether it also holds for code. Both methods are evaluated at all four checkpoints under the current protocol. The best observed 1.7B ALFWorld scores and training times appear in \cref{fig:checkpoint_accuracy}. Code uses eight completions per problem on the full $1{,}055$-problem LiveCodeBench set, with Pass@$1$ and Pass@$8$ computed from the same samples. 

The 1.7B mathematics comparison shows the same ordering as at 4B (\cref{tab:benchmarks}). HP-RLVR has the higher best observed three-benchmark score, $38.90\%$ against $37.80\%$, and the higher Pass@$12$, $66.67\%$ against $65.56\%$.

The code ordering reverses (\cref{tab:benchmarks}): GRPO has the higher best observed Pass@$1$ and Pass@$8$ at all three sizes. This contrast with mathematics limits the accuracy claim to the tested task and model setting. Both metrics use temperature $0.2$.

\section{Additional Efficiency Measurements}
\label{app:efficiency_detail}
\subsection{Current run timings and signal records}

\Cref{tab:zerograd} preserves the logged-record denominators of \cref{fig:training_signal}. Completed Qwen3-8B timings appear in the main-text \cref{tab:efficiency_scales}.

\begin{table}[t]
\centering
\small
\setlength{\tabcolsep}{2.5pt}
\caption{ALFWorld zero gradient-norm records, counted with the same parser. Fractions use all logged gradient records, rather than distinct optimizer steps. The figure reports the same fractions.}
\label{tab:zerograd}
\begin{tabular}{@{}l l c c@{}}
\toprule
Backbone & Method & \makecell{Zero-gradient\\records} & Fraction \\
\midrule
\multirow{2}{*}{Qwen3-4B} & GRPO ($G{=}8$) & $1{,}207 / 2{,}792$ & $43.2\%$ \\
 & HP-RLVR ($G{=}1$) & $\mathbf{0 / 353}$ & $\mathbf{0.0\%}$ \\
\addlinespace
\multirow{2}{*}{Qwen3-8B} & GRPO ($G{=}8$) & $1{,}343 / 2{,}739$ & $49.0\%$ \\
 & HP-RLVR ($G{=}1$) & $\mathbf{0 / 293}$ & $\mathbf{0.0\%}$ \\
\bottomrule
\end{tabular}
\end{table}

\paragraph{Timing coverage.}
Run-completion and accounting qualifications are reported alongside the main-text \cref{tab:efficiency_scales}. The 4B timing profile is not a completed paired $100$-step run.

\FloatBarrier
\section{LLM Usage Disclosure}
\label{app:llm_disclosure}
In answer to the submission questions on the use of large language models, a large language model was used to aid and polish the writing of this paper, for retrieval and discovery such as finding related work, and to draft and revise sections of the manuscript. It also generated the plotting code for \cref{fig:single_rollout_cost} and checked internal consistency between the text, the tables, and the figures. All research decisions, methods, implementations, experiments, and analysis are the authors' own work. Every numerical result reported here was produced by the authors' training and evaluation code, and every value was verified by the authors against the run artifacts before submission. The model did not generate any cited reference: every candidate reference it surfaced was checked against the bibliographic record and against the paper itself before being cited.

\end{document}